%% file: main.tex
\documentclass{article}

\input{packages}

\usepackage[final]{ewrl_2026}

\input{commands}

\input{math_commands.tex}

\title{Robust Successor Features}

\author{%
  Erik Nikulski \\
  Department of Engineering \\
  Universitat Pompeu Fabra \\
  \texttt{erik.nikulski@upf.edu}
  \And
  Yamen Habib \\
  Department of Engineering \\
  Universitat Pompeu Fabra \\
  \texttt{yamen.habib@upf.edu}
  \And
  Vicen\c{c} Gomez \\
  Department of Engineering \\
  Universitat Pompeu Fabra \\
  \texttt{vicen.gomez@upf.edu}
  \And
  Anders Jonsson \\
  Department of Engineering \\
  Universitat Pompeu Fabra \\
  \texttt{anders.jonsson@upf.edu}
  \And
  Rub\'en Moreno-Bote \\
  Department of Engineering \\
  Universitat Pompeu Fabra \\
  \texttt{ruben.moreno@upf.edu}
  \And
  Javier Segovia-Aguas \\
  Department of Engineering \\
  Universitat Pompeu Fabra \\
  \texttt{javier.segovia@upf.edu}
}

\date{}

\begin{document}

\maketitle

\input{0_abstract}
\input{1_introduction}
\input{2_preliminaries}

\input{3_robust_successor_features}
\input{4_training_usfa}

\input{5_experiments}
\input{7_discussion}

\input{acknowledgements}

\bibliography{references.bib}
\bibliographystyle{plainnat}

\newpage
\appendix

\pagebreak
\input{app_C_generalization}
\input{6_related_work}
\input{app_D_few_shot_experments}


\end{document}

%% file: packages.tex
\usepackage{amsmath} 
\usepackage{amsfonts} 
\usepackage{amssymb} 
\usepackage{tikz} 
\usepackage{hyperref} 
\usepackage{listings} 
\usepackage{xcolor} 
\usepackage{caption}
\usepackage{subcaption}
\usepackage{algorithm}
\usepackage{algpseudocode}


%% file: commands.tex
\newtheorem{theorem}{Theorem}
\newtheorem{proposition}{Proposition}
\newtheorem{definition}{Definition}
\newtheorem{lemma}{Lemma}
\newtheorem{remark}{Remark}

%% file: math_commands.tex
\usepackage{amsmath,amsfonts,bm}

\def\eqref#1{equation~\ref{#1}}

\def\1{\bm{1}}

\DeclareMathAlphabet{\mathsfit}{\encodingdefault}{\sfdefault}{m}{sl}
\SetMathAlphabet{\mathsfit}{bold}{\encodingdefault}{\sfdefault}{bx}{n}



%% file: 0_abstract.tex
\begin{abstract}
Generalization in Reinforcement Learning (RL) refers to the ability to execute close-to-optimal policies in unseen tasks after the agent has been trained on a different set of tasks. 
Building on the seminal work of the {\em successor representation} and further adaptations with function approximation, Transfer in RL has traditionally focused on generalizing to tasks that only differ in the reward function.
A decade after the introduction of the successor representation, Robust RL emerged simultaneously from several articles in the field of operations research. 
In Robust RL, the transition kernel is unknown, and the goal is to maximize the expected reward under this uncertainty.
Our work unifies these two paradigms through {\em robust successor features}, which generalize across both the reward function and the transition kernel, under the assumption that tasks are linear {\em Markov Decision Processes}.
We derive a bound on {\em Generalized Policy Improvement} (GPI) that explicitly quantifies how performance degrades with the mismatch between transition kernels, recovering existing successor-feature guarantees when dynamics are shared.
Finally, the generalization capabilities of robust successor features are validated on a grid-based benchmark and compared to previous alternatives that focus solely on either the reward or the transition kernel.
\end{abstract}

%% file: 1_introduction.tex
\section{Introduction}
\label{sec:intro}

Reinforcement learning (RL) agents \citep{sutton1998reinforcement} that perform well on the task they were trained on often fail when that task shifts even slightly, whether the shift is in what the agent is rewarded for or in how the environment responds to its actions.
Deep RL policies that master Atari games \citep{mnih2015human}, for instance, break when the paddle in Breakout is shifted by just a few pixels \citep{kansky2017schema}.
Two lines of research have each addressed one half of this problem in isolation.

In {\em Transfer RL}, the goal is a policy that easily adapts to changes in the reward signal $\mathcal{R}$.
Since learning a new policy from scratch for every reward variation is sample-inefficient, the aim is instead to reuse experience across the space of possible rewards and adapt to a new one zero-shot, or from a handful of samples if the new reward is unknown.
A successful attempt for tabular cases was the {\em successor representation} (SR) \citep{dayan1993improving}, which linearly approximates the reward with the future occupancies over states.
{\em Successor features} (SFs) \citep{barreto2017successor} extended the SR framework with features $\phi$ defined over transition triplets, turning the problem into maximizing a linear combination of future occupancies over features; this decoupling of task-dependent weights from general future occupancies is what enables zero-/few-shot adaptation.
{\em Universal successor features approximators} \citep{borsa2018universal} then used deep RL to approximate SFs across policy spaces, extending the mechanism to environments with high-dimensional inputs (e.g., pixel-based observations).

In {\em Robust RL} \citep{morimoto2005robust, iyengar2005robust, nilim2005robust}, the reward signal is fixed, but the transition kernel is not a single fixed model; instead it is only assumed to lie within an uncertainty (or ambiguity) set of plausible kernels, reflecting model misspecification, a mismatch between training and deployment dynamics, or even an adversary perturbing the dynamics \citep{pinto2017robust}. The policy must therefore be robust to the worst-case or expected transition kernel within this set, generalizing across the dynamics of the system rather than across its reward.

Despite this mirrored structure, the two lines of work remain largely disjoint: transfer methods that generalize across rewards typically assume a shared transition kernel, while robust RL methods that generalize across dynamics typically assume a fixed reward.
To the best of our knowledge, no method in the successor-feature and GPI setting offers a suboptimality guarantee when the reference and target tasks differ in both their reward and their transition kernel. 

In this work, we close this gap by unifying transfer RL and robust RL into a single framework, {\em robust successor features} (RSF), under the assumption of feature-based linear MDPs \citep{yang2019sample, jin2020provably}. RSF factorizes the value function into a task-agnostic representation and task-specific linear weights, one set for the reward and one for the transition kernel. Combined with a {\em Generalized Policy Improvement} (GPI) policy \citep{barreto2017successor}, this single learned representation generalizes zero-shot across changes in \emph{both} axes at once. Our main contributions are as follows:
\begin{itemize}
\item {\bf Robust successor features}, a formal characterization of transfer RL via successor features for linear MDPs with independent feature spaces for the reward and the transition kernel.
\item A {\bf theoretical guarantee} extending the GPI bound of \citet{barreto2017successor} (their Theorem 2) to reference and target tasks that differ in their transition kernels, in which the dynamics mismatch enters through an explicit term $\kappa_{\bf pq} = \min\!\bigl(2,\, L_\varphi \|{\bf p}-{\bf q}\|\bigr)$ that vanishes as all the reference dynamics approach the target ones, recovering the reward-only bound of \citet{barreto2017successor} as a special case.
\item A {\bf DDQN-based algorithm} for learning robust SFs, validated on grid-based experiments where RSF generalizes to unseen combinations of reward and dynamics with $93.6\%$ policy accuracy, outperforming both a USFA baseline ($86.7\%$) and ablations of our architecture trained under reward-only ($86.8\%$) or dynamics-only ($91.3\%$) variation.
\end{itemize}

{\bf Structure}. This article is organized as follows: we first introduce the formalism of MDPs, RL, and SFs. Then, we define our {\em robust successor features} that approximate the SFs for any linear MDP that shares features for the reward and the transition kernel. We continue with our main algorithm and experiments based on a 2D grid-based environment, and conclude with a discussion and future work.

{\bf Notation}. For a finite set ${\cal X}$, we use $\Delta({\cal X})$ to denote an arbitrary probability distribution over ${\cal X}$, and ${\cal U}({\cal X})$ for the uniform probability distribution.
Column vectors are formatted in bold, for example, $\textbf{v} \in \mathbb{R}^d$ is a $d$-dimensional vector of real numbers; and $\textbf{v}^\top$ stands for the transpose of ${\bf v}$ (i.e., a row vector).
We use $||\cdot||_1$, $||\cdot||$ and $||\cdot||_\infty$, for the $l_1$, $l_2$, and $l_\infty$-norms, respectively.

%% file: 2_preliminaries.tex
\section{Preliminaries}

This section covers the preliminary work for transfer and robust reinforcement learning that are closest to our contributions, as well as the main framework for modeling and formalizing these problems. Other related work is expanded in Appendix~\ref{sec:related_work}.

\subsection{Markov Decision Process (MDP)}
\label{}
A discounted {\em Markov Decision Process} (MDP) \citep{puterman2014markov} is a model 
${\cal M}= \langle {\cal S}, {\cal A}, {\cal P}, {\cal R}, \gamma\rangle$,  
where ${\cal S}$ is the state space, ${\cal A}$ is the action space, and ${\cal P}:{\cal S}\times{\cal A} \rightarrow \Delta({\cal S})$ is the transition kernel with $\Delta$ denoting the simplex over the state space.
Thus, taking an action in a given state induces a probability distribution over the subsequent states.
The reward function ${\cal R} : {\cal S}\times{\cal A}\times{\cal S} \rightarrow \mathbb{R}$ provides a scalar reward for each transition $s\xrightarrow{a}s'$, and $\gamma\in[0,1)$ is the discount factor.

A solution to an MDP is a policy $\pi: {\cal S}\rightarrow {\cal A}$ mapping every state $s\in{\cal S}$ to an action $a\in{\cal A}$ that maximizes the expected (discounted) reward. 
The quality function for a given policy $\pi$, when applying action $a$ in state $s$, is given as

\begin{equation}
Q^\pi(s,a) = \mathbb{E}^\pi\left[\sum^\infty_{t=0} \gamma^t{\cal R}(S_t,A_t,S_{t+1}) \middle| S_0=s,A_0=a\right]
\label{eq:q_values}
\end{equation}

At any state $s_t$ at time $t$, the {\em greedy} policy chooses the action $a_t$ with highest expected return: $$a_t\leftarrow \pi(s_t) = \arg\max_{a\in{\cal A}} Q^\pi(s_t,a).$$ 
Alternatively, a \emph{stochastic} policy samples actions from a probability distribution over the action space, given the current state.
We refer to $\pi(a|s)$ as the probability of choosing action $a\in{\cal A}$ in a state $s\in{\cal S}$. 
For instance, a {\em softmax} policy  $$a_t\sim \pi(\cdot|s_t) = \frac{e^{Q(s_t,\cdot)/\beta}}{\sum_{a\in{\cal A}} e^{Q(s_t,a)/\beta}},$$ draws actions from the exponentiated values, which are balanced by a temperature parameter $\beta$. If $\beta \rightarrow \infty$, the policy will be a uniform distribution over the action space, while $\beta \rightarrow 0$ recovers the greedy policy.


\subsection{Universal Successor Features (USFs)}


\citet{barreto2017successor} express the expected immediate reward in any given transition as a linear combination of features $\phi\in\mathbb{R}^d$ over the transition, and the task weights $\textbf{w}\in\mathbb{R}^d$. Formally, the expected immediate reward is

\begin{equation}
\mathbb{E}\left[{\cal R}_\textbf{w}(s,a,s')\right] = r_\textbf{w}(s,a,s') = \phi(s,a,s')^\top \textbf{w} \quad \forall s,s'\in{\cal S}, a\in{\cal A}.
\label{eq:linear_reward}
\end{equation}

The linearized reward function from Eq.~\ref{eq:linear_reward} updates the definition of Q-values in Eq.~\ref{eq:q_values} as follows,

\begin{align}
\nonumber Q^\pi_\textbf{w}(s,a) &= \mathbb{E}^\pi\left[\sum^\infty_{t=0} \gamma^t{\cal R}_w(S_t,A_t,S_{t+1}) \middle| S_0=s,A_0=a\right]\\
\nonumber &= \mathbb{E}^\pi\left[\sum^\infty_{t=0} \gamma^t \phi(S_t,A_t,S_{t+1})^\top\textbf{w} \middle| S_0=s,A_0=a\right]\\
&= \mathbb{E}^\pi\left[\sum^\infty_{t=0} \gamma^t \phi(S_t,A_t,S_{t+1}) \middle| S_0=s,A_0=a\right]^\top \textbf{w} \quad \forall s\in{\cal S}, a\in{\cal A}.
\label{eq:q_w_values}
\end{align}

The expectation in Eq.~\ref{eq:q_w_values} denotes the future discounted aggregation of features, which is simply referred to as {\em successor features} (SFs) and represented as

\begin{equation}
\psi^\pi(s,a) = \mathbb{E}^\pi\left[\sum^\infty_{t=0} \gamma^t \phi(S_t, A_t, S_{t+1}) \middle| S_0=s, A_0=a\right] \quad \forall s\in{\cal S}, a\in{\cal A}.
\label{eq:successor_features}
\end{equation}

Note that $\phi$ in Eq.~\ref{eq:successor_features} is playing the role of the reward function like ${\cal R}$ in Eq.~\ref{eq:q_values}. Thus, the Q-values can be expressed in matrix form as

\begin{equation}
Q^\pi_\textbf{w} = \left(\psi^\pi\right)^\top \textbf{w}.
\end{equation}

\citet{borsa2018universal} introduce the concept of {\em universal successor features}, which extends the previous equations by also parameterizing the Q-values with a policy encoding function $e: ({\cal S}\rightarrow {\cal A})\rightarrow \mathbb{R}^d$, that maps a policy $\pi$ to a vector $\textbf{z} = e(\pi) \in \mathbb{R}^d$. In this setting, $\psi^\pi(s,a) \equiv \psi(s,a,\textbf{z})$, hence $Q^\pi_\textbf{w}(s,a) \equiv Q(s,a,\textbf{w},\textbf{z})$, where we refer to $\textbf{z}$ as the policy for short, and Q-values are updated as follows

\begin{align}
\nonumber Q(s,a,\textbf{w},\textbf{z}) &= \psi(s,a,\textbf{z})^\top \textbf{w}\\
&= \mathbb{E}^\textbf{z}\left[\sum^\infty_{i=t} \gamma^{i-t}\phi(S_i,A_i,S_{i+1}) \middle| S_t=s, A_t=a\right]^\top \textbf{w} \quad \forall s\in{\cal S}, a\in{\cal A}, \textbf{z}, \textbf{w} \in \mathbb{R}^d.
\label{eq:usf_q_values}
\end{align}

Note that $z \leftarrow w$ makes $Q(s,a,w,z=w) \equiv Q(s,a,w)$, which is equivalent to the problem of {\em Universal Value Functions} (UVFs) \citep{schaul2015universal}, where the chosen policy is the one that solves the input task optimally.

This setting can exploit function approximation over the policies in $\psi$, where successor features generalize across the policy space, and new tasks can be solved in a zero-shot fashion. This is guaranteed when combined with {\em Generalized Policy Improvement} (GPI) which is defined for a new task $\textbf{w}'$ as

\begin{equation}
\pi(s) \in \arg\max_{a\in{\cal A}} \max_{\textbf{z}\in{\cal C}} \tilde{Q}(s,a,\textbf{w}',\textbf{z}) = \arg\max_{a\in{\cal A}}\max_{\textbf{z}\in{\cal C}} \tilde{\psi}(s,a,\textbf{z})^\top \textbf{w}' \quad \forall s \in {\cal S},
\end{equation}

where ${\cal C}\subset \mathbb{R}^d$ is a subset of tasks that generalizes UVFs, SF\&GPI, and their approximated counterparts.

\subsection{Robust Reinforcement Learning}

A task in Robust RL \citep{morimoto2005robust, iyengar2005robust, nilim2005robust} is often modeled as a finite robust MDP \citep{behzadian2021fast}.
This is defined as an MDP $\mathcal{M}$ with an unknown transition kernel ${\cal P} \in \Delta({\cal S})^{{\cal S}\times{\cal A}}$ meaning the environment's dynamics are not fixed but instead belong to a set of possible distributions.

The goal is to compute a policy $\pi : {\cal S} \rightarrow \Delta({\cal A})$ from a set of stationary policies $\Pi$ that maximizes the expected discounted cumulative reward for the worst-case transition probabilities ${\cal P}$, that is,
\begin{equation}
\max_{\pi\in\Pi}\min_{{\cal P}\in \Delta({\cal S})^{{\cal S}\times{\cal A}}}  \mathbb{E}^\pi_{\cal P} \left[\sum^\infty_{t=0} \gamma^t {\cal R}(S_t,A_t,S_{t+1}) \middle| S_0 \sim \nu_0\right],
\end{equation}

where the first state is sampled from the initial state distribution $\nu_0$, the actions are drawn from the policy $\pi$, and the next state depends on the given transition kernel ${\cal P}$. 

\textsc{Poly}-time approximations of the optimization problem exist to overcome the intractability of the general case \citep{wiesemann2013robust}. 
For instance, {\em ambiguity sets} are common in the RRL literature, where a transition kernel is estimated from samples and the environment considers each state-action pair $\bar{\cal P}_{s,a}$ (resp., state $\bar{\cal P}_{s}$) as an independent problem to optimize, choosing for each a probability distribution over the state space ${\cal P}_{s,a} \in \Delta({\cal S})$, such that, $||\bar{\cal P}-{\cal P}||_\infty$ is upper-bounded by a robustness budget $\kappa \geq 0$. 
More formally, an ambiguity set is ${\cal T} = \{{\cal P}\in \Delta({\cal S})^{{\cal S}\times {\cal A}} | {\cal P}_{s,a}\in \Delta({\cal S}), ||\bar{\cal P}_{s,a}-{\cal P}_{s,a}||_\infty \leq \kappa_{s,a}, s\in{\cal S}, a\in{\cal A}\}$, where each $(s,a)$ entry is restricted to a $\kappa_{s,a}$. Therefore, the optimization problem restricted to an ambiguity set is
\begin{equation}
\max_{\pi\in\Pi}\min_{{\cal P}\in {\cal T}}  \mathbb{E}^\pi_{\cal P} \left[\sum^\infty_{t=0} \gamma^t {\cal R}(S_t,A_t,S_{t+1}) \middle| S_0 \sim \nu_0\right].
\end{equation}

However, worst-case analysis is overly pessimistic as it assumes that the environment is adversarial to the agent. Less conservative approaches focus on the $\alpha$-quantile subset of the uncertainty set of transition kernels, solving the problem in expectation \citep{lin2022bayesian,xie2022robust}. These are closer to our work, where generalization to diverse environments is formalized in expectation over the ambiguity set,

\begin{equation}
\label{eq:our_rrl}
\max_{\pi} \mathbb{E}^\pi_{{\cal P}\sim{\cal T}} \left[\sum^\infty_{t=0} \gamma^t {\cal R}(S_t, A_t, S_{t+1}) \middle| S_0\sim\nu_0\right].
\end{equation}


%% file: 3_robust_successor_features.tex
\section{Robust Successor Features}

Successor feature frameworks explore the generalization across tasks \citep{barreto2017successor} and policies \citep{borsa2018universal} with approximations of the quality function. 
However, one of the main limitations is that both frameworks assume that all tasks share the same transition probabilities. 
In this work, we relax this assumption so that tasks may differ not only in the reward function but also in the transition kernels, under the condition that both are representable as linear combinations of features.

Linear representations of the model were first introduced as linear MDPs~\citep{yang2019sample}, where the transition and the reward functions share a common feature space $\Phi : {\cal S}\times {\cal A} \rightarrow \mathbb{R}^d$ for every state and action pair. 
Because coupling both functions under the same feature space could be highly restrictive for representing diverse environments, we propose using independent feature representations for the reward and the transition kernel. 
Formally, the probability distribution over the next states for each state-action pair takes the following form

\begin{equation}
{\cal P}(s'|s,a) = \varphi(s,a,s')^\top\textbf{p} \quad \forall s,s'\in{\cal S}, a\in{\cal A},
\label{eq:linear_kernel}
\end{equation}

where $\varphi: {\cal S}\times{\cal A}\times{\cal S} \rightarrow \mathbb{R}^k$ is the feature function for every transition triplet, and $\textbf{p} \in \mathbb{R}^k$ denotes the linear weights that are specific to the task. 
Without loss of generality, we assume that $\mathbf{p}$ lies in the probability simplex. 
This assumption does not restrict the representational capacity of the model, as any arbitrary scaling of the weights can be seamlessly absorbed by an inverse scaling of the feature function $\varphi$.

This family of probability distributions is particularly useful for environments where transition dynamics systematically vary due to physical parameters, such as in robotic navigation across terrains with varying friction or drone flight under different wind conditions.
By letting $\varphi$ capture fundamental kinematics and $\textbf{p}$ represent specific environmental conditions, this approach enables rapid adaptation to new dynamics.

\begin{proposition}
Given a transition kernel matrix ${\cal P}$ and feature matrix $\varphi$, an exact solution for the weights that linearize the transition kernel is
$$\textbf{p} = \varphi^\dagger {\cal P}.$$
\label{prop:exact_p}
\end{proposition}


Our main objective now is to decouple as much as possible the dynamics of an MDP
from the task and the policy, such that {\em universal successor features} $\psi\in\mathbb{R}^d$ can be 
expressed as a linear combination of what we name {\em robust successor features} 
$\tau\in\mathbb{R}^{k\times d}$ and the linear weights of the dynamics $\textbf{q}\in\mathbb{R}^k$ (playing the role for the dynamics that $\textbf{z}$ plays for the policy in Eq.~\ref{eq:usf_q_values}), 
that is, $\psi(s,a,\textbf{z}) = \tau(s,a,\textbf{z},\textbf{q})^\top\textbf{q}$. 
Thus, we start from the recursive definition of the Bellman equation for 
{\em universal successor features} in Eq.~\ref{eq:successor_features},
which is given by
\begin{align}
\psi(s,a,\textbf{z}) &= \sum_{s'} {\cal P}(s'|s,a) \left[ \phi(s,a,s') + \gamma \sum_{a'} \pi_{\textbf{z}}(a'|s') \psi(s',a',\textbf{z}) \right].
\label{eq:bellman_usf}
\end{align}

By substituting $\psi(s,a,\textbf{z})=\tau(s,a,\textbf{z},\textbf{q})^\top \textbf{q}$ and ${\cal P}(s'|s,a) = \varphi(s,a,s')^\top \textbf{q}$ we obtain
\begin{align*}
\tau(s,a,\textbf{z},\textbf{q})^\top \textbf{q} &= \sum_{s'} \left[ \phi(s,a,s') + \gamma \sum_{a'} \pi_{\textbf{z},\textbf{q}}(a'|s') \tau(s',a',\textbf{z},\textbf{q})^\top \textbf{q} \right] \varphi(s,a,s')^\top \textbf{q}.
\end{align*}
This yields a recursive equation for estimating $\tau(s,a,\textbf{z},\textbf{q})$.
This implies that if $\textbf{q}$ matches the exact solution from 
Prop.~\ref{prop:exact_p}, then 
$\tau(s,a,\textbf{z},\textbf{q})^\top\textbf{q} \equiv \psi(s,a,\textbf{z})$. 
However, it can also generalize zero-shot to any novel transition dynamics ${\cal P}'$ expressible as a linear combination of the features. 
Also, note that we define $\tau(s,a,\textbf{z},\textbf{q})$ as the solution of

\begin{align}
\tau(s,a,\textbf{z},\textbf{q})^\top &= \sum_{s'} \left[ \phi(s,a,s') + \gamma \sum_{a'} \pi_{\textbf{z},\textbf{q}}(a'|s') \tau(s',a',\textbf{z},\textbf{q})^\top \textbf{q} \right] \varphi(s,a,s')^\top.
	\label{eq:tau_explicit}
\end{align}

{\em Robust successor features} extends the state-action value functions from Eq.~\ref{eq:usf_q_values} to

%

\begin{align}
Q(s,a,\textbf{w},\textbf{p},\textbf{z},\textbf{q}) := \textbf{p}^\top \tau(s,a,\textbf{z},\textbf{q})\,\textbf{w},
\label{eq:linear_q_values}
\end{align}
where $Q(s,a,\textbf{w},\textbf{p},\textbf{z},\textbf{q}) = Q(s,a,\textbf{w},\textbf{z})$ if both share the same transition kernel ($\textbf{p}=\textbf{q}$).
For $\textbf{p}\neq\textbf{q}$, $\tau$ learned under reference dynamics $\textbf{q}$ can be queried against a different task's dynamics $\textbf{p}$; this is exact only for the first transition, and the resulting error is bounded in Lem.~\ref{lem:cross_dynamics}.

Given a set of joint available rewards and dynamics ${\cal C}$, which might be induced by the set of input tasks, a deterministic policy for a
new task with the reward $\textbf{w}'$ and the transition dynamics $\textbf{p}'$ is 
\begin{align}
    \nonumber \pi(s) &\in \arg\max_{a\in{\cal A}}\max_{\textbf{z},\textbf{q}\in{\cal C}} Q(s,a,\textbf{w}',\textbf{p}',\textbf{z},\textbf{q})\\
    &= \arg\max_{a\in{\cal A}}\max_{\textbf{z},\textbf{q}\in{\cal C}} (\textbf{p}')^\top \tau(s,a,\textbf{z},\textbf{q}) \textbf{w}'.
    \label{eq:robust_gpi}
\end{align}

%% file: 4_training_usfa.tex
\section{Learning Robust Successor Features}


Generalization to environments that may differ in their reward signals and transition kernels, using the same predictive function $\tau$ is at the core of robust SFs. In this section, we show how to obtain $\tilde{\tau}$, an approximation of the optimally robust SFs, and analyze the approximation error of  $\tilde{Q}$ induced by $\tilde{\tau}$ when following the GPI policy (Eq.~\ref{eq:robust_gpi}).  




Let us first define
\begin{align*}
{\cal M}^{\phi,\varphi} := \Big\{\, M = \langle {\cal S},{\cal A},{\cal P},{\cal R},\gamma\rangle \;\Big|\; &\exists\, \mathbf{w}\in\mathbb{R}^d,\ \mathbf{p}\in\mathbb{R}^k \text{ with } \mathbf p \geq 0,\ \|\mathbf p\|_1 = 1, \\
&\text{s.t. } \forall s,s'\in{\cal S}, a\in{\cal A}: \\
&{\cal R}(s,a,s') = \phi(s,a,s')^\top \mathbf{w}, \ {\cal P}(s'|s,a) = \varphi(s,a,s')^\top \mathbf{p} \,\Big\},
\end{align*}
as the set of MDPs that respect the linear decomposition of the reward and the transition kernel, with dynamics weights $\mathbf p$ restricted to the probability simplex.
As before, this restriction is without loss of generality, since any rescaling of $\mathbf p$ can be absorbed into $\varphi$.
Let $\phi_{\max} := \max_{s,a,s'}\|\phi(s,a,s')\|$ bound the reward features and let $\kappa_{\mathbf{pq}} := \min(2, L_\varphi\|\mathbf p - \mathbf q\|)$, with $L_\varphi$ as in Def.~\ref{def:Lphi}, bound the $\ell_1$ distance $\|P_{\mathbf{p}}(\cdot|s,a) - P_{\mathbf{q}}(\cdot|s,a)\|_1$ between the transition kernels of $M_{\mathbf{w,p}}$ and $M_{\mathbf{z,q}}$ (Lem.~\ref{lem:kernel_l1_bound}, App.~\ref{sec:theoretical_results}).
Then, we extend the GPI theorem for successor features (\citet{barreto2017successor}, Thm.~2) to robust SFs on ${\cal M}^{\phi,\varphi}$. 
Compared to the reward-only bound, two dynamics terms appear (proof in App.~\ref{sec:theoretical_results}):
a kernel-mismatch term for the best reference task, scaled by the magnitude of its policy encoding $\mathbf z$ (generalizing the bound of \citet{borsa2018universal}),
and a cross-dynamics term, which arises because $\mathbf p^\top\tau\,\mathbf w$ follows the target dynamics only for the first transition (Lem.~\ref{lem:cross_dynamics}).

\begin{theorem}
\label{thm:gpi_robust_sf}
    Let $Q^\pi_{{\bf w,p}}$ (resp., $Q^{\pi_{\bf z,q}}_{{\bf w,p}}$) be the action-value function of applying policy $\pi$ (resp., optimal policy $\pi_{\bf z,q}$ of $M_{\mathbf{z,q}}$) on task $M_{\bf w,p}\in \mathcal{M}^{\phi,\varphi}$, and $Q^*_{\mathbf{w,p}}$ its optimal action-value function.
    Given a set of tasks ${\cal C}\subset{\cal M}^{\phi,\varphi}$ with approximations
    $\tilde{Q}^{\pi_{\bf z,q}}_{\bf w,p}(s,a) =
    {\bf p}^\top \tilde{\tau}(s,a,{\bf z},{\bf q}) {\bf w}$
    for $({\bf z,q}) \in \mathcal{C}$,
    and the GPI policy
    $\pi(s) \in \arg\max_{a\in \mathcal{A}}\max_{({\bf z,q})\in\mathcal{C}}
    {\bf p}^\top \tilde{\tau}(s,a,{\bf z},{\bf q}) {\bf w}$,
    let $\rho_{\cal C} := \max_{s,a,({\bf z,q})\in{\cal C}} \|\tau(s,a,{\bf z},{\bf q}) - \tilde\tau(s,a,{\bf z},{\bf q})\|_2$
    denote the worst-case approximation error of $\tilde\tau$ over ${\cal C}$, where $\tau$ is the solution of Eq.~\ref{eq:tau_explicit} for $\pi_{\mathbf{z,q}}$ and $\|\cdot\|_2$ is the spectral norm.
    Then, for all states $s\in {\cal S}$ and actions $a \in {\cal A}$,
    \begin{align*}
      Q^*_{{\bf w,p}}(s,a) - Q^\pi_{{\bf w,p}}(s,a) &\leq \frac{2}{1-\gamma}\Bigg(\phi_{\max}\min_{({\bf z,q})\in{\cal C}}\Big(\|{\bf w} - {\bf z}\| + \frac{\kappa_{\mathbf{pq}}}{1-\gamma}\|{\bf z}\|\Big) \\
      &\qquad + \frac{\gamma\,\phi_{\max}\|\mathbf w\|}{(1-\gamma)^2}\max_{(\mathbf{z,q})\in\mathcal C}\kappa_{\mathbf{pq}} + \|\mathbf w\|\,\rho_{\mathcal C}\Bigg).
    \end{align*}
\end{theorem}

Note that the bound collapses to the reward-only case when $\mathbf{q}=\mathbf{p}$ for all $(\mathbf{z,q}) \in \mathcal{C}$.
While the second term only needs one reference task with dynamics close to $\mathbf{p}$, the third term needs all library dynamics close to $\mathbf{p}$, because any library task can determine the GPI action.

Having established GPI as a valid protocol for approximating the optimal action-value function, the next step is to compute $\tilde{Q}$. 
We employ the $1$-step \emph{Temporal Difference} (TD) error \citep{sutton1998reinforcement} for learning $\tilde{\tau}$, which satisfies the Bellman equation \citep{borsa2018universal}. 
Expanding the terms using the robust SF formulation (Eq.~\ref{eq:linear_q_values}) yields:

\begin{align}
\label{eq:td_error}
\nonumber \delta^{\bf wp}_{\bf zq} &= r_{\bf w}(s,a,s') + \gamma \tilde{Q}(s', \pi_{\bf z,q}(s'), {\bf w}, {\bf p}, {\bf z}, {\bf q}) - \tilde{Q}(s, a, {\bf w}, {\bf p}, {\bf z}, {\bf q}) \\ 
 &= \left[ \phi(s,a,s')^\top + {\bf p}^\top\left[ \gamma \tilde{\tau}(s', \pi_{\bf z,q}(s'), {\bf z}, {\bf q}) - \tilde{\tau}(s, a, {\bf z}, {\bf q}) \right] \right]{\bf w}  = \delta^{\bf p}_{\bf zq}  {\bf w}.
\end{align}

The process of learning robust SFs assumes a set of training tasks ${\cal M}^{\phi,\varphi}$, with known model decompositions $(\textbf{w}, \textbf{p})$. 
During training, a task $M_{\textbf{w}, \textbf{p}}$ can be drawn uniformly at random from ${\cal M}^{\phi,\varphi}$.
To prevent ${\tilde{\tau}}$ from overfitting and to ensure bounded updates, it is trained following the GPI policy over a set of tasks ${\cal C}$.
These tasks are sampled from a multivariate Gaussian distribution ${\cal G}$ conditioned on $M_{\textbf{w},\textbf{p}}$. 
At each step, the agent executes an action $a$ selected via an $\epsilon$-greedy strategy over the GPI policy.
Subsequently, the policies for each task in ${\cal C}$ are updated, the corresponding TD-errors $\delta^{\textbf{p}}_{\textbf{z}_i,\textbf{q}_i}$ calculated, and these errors are used to update the parameters $\theta$ of $\tilde{\tau}$.


\begin{algorithm}
\caption{Learning Robust SFs with $\epsilon$-greedy Q-learning}
\label{alg:learn_robust_sf}
    \begin{algorithmic}[1]
    	\Require ${\cal M}^{\phi,\varphi}$ (training tasks), ${\cal G}$ (multivariate Gaussian distribution), $N$ (number of sampled tasks), $\epsilon$  (exploration parameter), $\alpha$ (learning rate)
    	\State select initial state $s\in{\cal S}$
    	\For{number of steps}
    		\State $M_{\textbf{w}, \textbf{p}} \sim {\cal U}({\cal M}^{\phi,\varphi})$
    		\State ${\cal C} \leftarrow \{(\textbf{z}_i, \textbf{q}_i) \sim {\cal G}(\cdot | \textbf{w}, \textbf{p}) \textbf{ } \forall i \in [1, N] \}$
    		\State \textbf{if} Bernoulli($\epsilon$)$=1$ \textbf{then} $a\sim {\cal U}({\cal A})$ \{Exploration\}
    		\State \textbf{else} $a \leftarrow \arg\max_b\max_i\textbf{p}^\top\tilde{\tau}(s,b,\textbf{z}_i,\textbf{q}_i)\textbf{w}$ \{GPI\}
    		\State Apply $a$ and observe transition $(s,a,s',\phi,\varphi)$ 
    		\For {$i\leftarrow 1, \ldots, N$} \{Update $\tilde{\tau}$\}
			\State $\pi_{{\bf z}_i,{\bf q}_i}\leftarrow \arg\max_b \textbf{q}_i^\top\tilde{\tau}(s',b,\textbf{z}_i,\textbf{q}_i)\textbf{z}_i$  
			\State $\theta \xleftarrow{\alpha} \delta^{\bf p}_{{\bf z}_i,{\bf q}_i} \nabla_\theta \tilde{\tau}$  \{Update $\theta$ with the TD-error in Eq.~\ref{eq:td_error} \}
    		\EndFor
    		\State $s \leftarrow s'$  \{Restart to an initial state if $s'$ is terminal\} 
    	\EndFor
    \State \textbf{return} $\tilde{\tau}$
    \end{algorithmic}
\end{algorithm}

Algorithm~\ref{alg:learn_robust_sf} can be further extended to incorporate deep RL \citep{barreto2018transfer}. 
In our case, we use Double DQN (DDQN) \citep{van2016deep}, an extension of DQN \citep{mnih2015human}, to learn $\tilde{Q}^{\textbf{z},\textbf{q}}_{\textbf{w},\textbf{p}}$ in an {\em off-policy} setting where observed transitions and TD-errors from different MDPs and policies are stored in a replay buffer to subsequently update $\tilde{\tau}$.
Our framework can also be easily adapted to learn $(\textbf{w}, \textbf{p})$ with any linear regression algorithm when the model decomposition is unknown. 
Learning the feature map $\phi$ is also common in previous work (e.g., \citet{barreto2018transfer}); however, this is beyond the scope of this article, which focuses on the generalization capabilities of robust SFs when compared to alternatives that focus solely on either the reward function or the transition kernel. 
Learning both $\phi$ and $\varphi$ from a given set of MDPs represents a promising research direction that would pave the way for addressing continuous RL tasks.

%% file: 5_experiments.tex
\section{Experiments}

This section analyzes the generalization capabilities of robust successor features. 
First, we outline the experimental setup in a classic stochastic grid-based domain (Section 5.1). 
Then, we evaluate environment-aware zero-shot generalization when true environment parameters are known (Section 5.2). 
Appendix \ref{sec:apdx-few-shot} provides an assessment of the few-shot adaptation in novel environments where the agent must estimate dynamics and rewards from limited interactions.

\subsection{Experimental Setup}

We demonstrate generalization over transition kernels using the classic Gridworld environment~\citep{russell2021artificial}. 
The environment is a $3\times 4$ grid where the agent starts in the bottom-left corner $(0,0)$. 
The grid contains a blocked location at $(1,1)$ and two terminal states in the upper-right region. 
Reaching a terminal state ends the episode; the upper terminal yields a reward of $+1$, and the lower yields $-1$. 
Additionally, every step incurs a constant negative reward (e.g., $r(s, a) = -0.02$). 
The agent can choose to move in any of the four cardinal directions (i.e., $\mathcal{A} = \{N, E, S, W\}$). 
Actions that would drive the agent off the grid simply cause it to remain in its current state.
The transition dynamics are stochastic: the agent moves in the intended direction with probability $p$, and is pushed to either of the two orthogonal directions with probability $p'=\frac{1-p}{2}$ (e.g., if $p=0.8$, then $p'=0.1$).

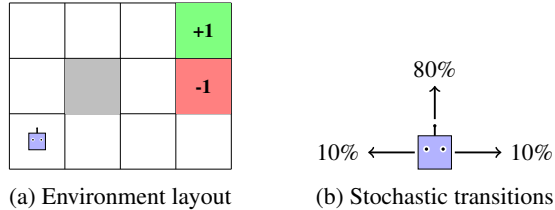
\begin{figure}[hbtp]
\centering
\begin{subfigure}{0.22\textwidth}
\resizebox{\textwidth}{!}{
\begin{tikzpicture}[scale=1]
  \foreach \x in {0,1,2,3,4}
    \draw (\x,0) -- (\x,3);
  \foreach \y in {0,1,2,3}
    \draw (0,\y) -- (4,\y);

  \fill[gray!50] (1.01,1.01) rectangle (1.99,1.99);

  \begin{scope}[shift={(0.5,0.5)}, scale=0.3]
    \draw[fill=blue!30] (-0.5,-0.5) rectangle (0.5,0.5);
    \fill[white] (-0.25,0.1) circle (0.1);
    \fill[white] (0.25,0.1) circle (0.1);
    \fill[black] (-0.25,0.1) circle (0.05);
    \fill[black] (0.25,0.1) circle (0.05);
    \draw[thick] (0,0.5) -- (0,0.8);
    \fill (0,0.8) circle (0.05);
  \end{scope}

  \fill[green!50] (3.01,2.01) rectangle (3.99,2.99);
  \node at (3.5,2.5) {\textbf{+1}};

  \fill[red!50] (3.01,1.01) rectangle (3.99,1.99);
  \node at (3.5,1.5) {\textbf{-1}};
\end{tikzpicture}
}
\caption{Environment layout}
\end{subfigure}
\hspace{0.05\textwidth}
\begin{subfigure}{0.25\textwidth}
\resizebox{\textwidth}{!}{
\begin{tikzpicture}[scale=1]

  \begin{scope}[shift={(0,0)}, scale=0.5]
    \draw[fill=blue!30] (-0.5,-0.5) rectangle (0.5,0.5);
    \fill[white] (-0.25,0.1) circle (0.1);
    \fill[white] (0.25,0.1) circle (0.1);
    \fill[black] (-0.25,0.1) circle (0.05);
    \fill[black] (0.25,0.1) circle (0.05);
    \draw[thick] (0,0.5) -- (0,0.8);
    \fill (0,0.8) circle (0.05);
  \end{scope}

  \draw[->, thick] (0,0.5) -- (0,1) node[above] {80\%};
  \draw[->, thick] (-0.3,0) -- (-1,0) node[left] {10\%};
  \draw[->, thick] (0.3,0.0) -- (1,0) node[right] {10\%};

\end{tikzpicture}
}
\caption{Stochastic transitions}
\end{subfigure}
\caption{The classic $3 \times 4$ Gridworld from \citet{russell2021artificial}.(a) The agent navigates toward terminal states while avoiding the blocked tile. (b) Actions are stochastic, with a probability of slipping orthogonally.}
\label{fig:aima_grid}
\end{figure}

In this environment, optimal behavior varies significantly based on the magnitude of the step penalty and the stochasticity of the transition dynamics.
For instance, high immediate step costs (e.g., $r(s, a) < -0.2$) drive the agent toward the nearest terminal state, regardless of whether the terminal reward is positive or negative.
If actions are highly deterministic, the agent selects the shortest path even if it borders the negative terminal state.
However, under highly stochastic conditions, a risk-averse policy emerges where the agent takes longer, safer routes.
Consequently, this setting effectively assesses the generalization capabilities of successor transition features when trained in specific configurations and evaluated on unseen ones.

We evaluate environment-aware zero-shot generalization and few-shot adaptation for four methods: our full model \textbf{RSF}, the \textbf{USFA} baseline~\citep{borsa2018universal}, and two ablations of RSF that restrict $\tau$ to generalizing over a single axis, \textbf{RSF-reward-only} and \textbf{RSF-prob-only}.
RSF and USFA are both trained on four environments spanning the Cartesian product of $p \in \{0.65, 0.95\}$ and $r \in \{-0.02, -0.2\}$, so that both methods observe identical training data.
RSF-reward-only is trained on two environments with a fixed step probability $p=0.8$ and $r \in \{-0.02, -0.2\}$; RSF-prob-only is trained on two environments with a fixed step reward $r=-0.1$ and $p \in \{0.65, 0.95\}$.
Each RSF variant (RSF, RSF-reward-only, RSF-prob-only) is trained with Double DQN (DDQN)~\citep{van2016deep} following Algorithm~\ref{alg:learn_robust_sf}; USFA follows the same training protocol, substituting its own $\psi$-based architecture for $\tau$.
For the first $500$ steps, the agent moves randomly to populate the replay buffer~\citep{lin1992self, mnih2015human}.
For the subsequent 1,500 steps, the agent strictly learns the vectors $\mathbf{p}$ and $\mathbf{w}$ (or, for USFA, $\mathbf{w}$ alone) via Mean Squared Error (MSE) using the Adam optimizer~\citep{kingma2015adam}, ensuring representation stability.
From step 2,000, training proceeds with Huber loss \citep{huber1992robust} and the Adam optimizer until 20,000 interactions.

\subsection{Environment-aware zero-shot generalization}

We evaluate the zero-shot transfer capabilities by measuring the Mean Absolute Error (MAE) of the predicted Q-values and the policy accuracy with respect to the optimal policy. 

To conduct this evaluation, the models for RSF, USFA, RSF-reward-only, and RSF-prob-only are first trained on their respective training environments (detailed above).
Following training, they are evaluated on a common, comprehensive grid of unseen configurations, $\mathcal{E} = \{0.65, 0.66, \dots, 0.95\} \times \{-0.20, -0.19, \dots, -0.01\}$, spanning all combinations of step probabilities $p$ and step rewards $r$. 
To purely isolate and evaluate the zero-shot generalization capabilities, the agent is directly provided with the true reward weights $\mathbf{w}$ and transition dynamics parameters $\mathbf{p}$ during testing.

Figure~\ref{fig:aima-experiments-mae} displays the mean absolute error (MAE) in Q-values predicted by the models compared to an optimal policy, averaged over five independent runs.
Formally, for a given configuration $c = (p, r) \in \mathcal{E}$, this error is computed over the state-action space as $\mathbb{E}_{s,a}[\lvert Q^*_c(s,a) - \hat{Q}_c(s,a)\rvert]$.

\begin{figure}[htbp]
    \centering
    \begin{subfigure}[c]{\textwidth}
        \includegraphics[width=\textwidth]{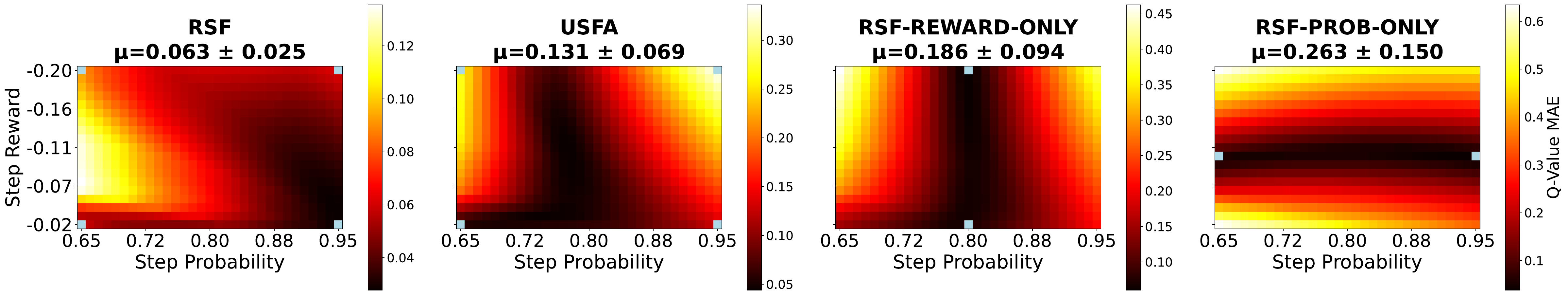}
    \end{subfigure}

    \caption{
    	Mean errors of the Q-function for the stochastic 2D grid experiments.
    	Note the different scales and that darker regions indicate lower error (better performance).
    	Training environments are marked by the light blue rectangles.
    	The mean errors for RSF, USFA, RSF-reward-only, and RSF-prob-only are $0.063$, $0.131$, $0.186$, and $0.263$, respectively.}
    \label{fig:aima-experiments-mae}
\end{figure}

\paragraph{RSF: Full Model}
Our full model is trained on four environments derived from the Cartesian product of $p \in \{0.65, 0.95\}$ and $r \in \{-0.02, -0.2\}$, exposing $\tau$ to variation in both reward and dynamics.
RSF achieves the lowest MAE across the evaluation grid ($\mu=0.063\pm0.025$), degrading only mildly toward the left side (low $p$), where reward and dynamics generalization must be traded off simultaneously.

\paragraph{USFA: Baseline}
The USFA baseline~\citep{borsa2018universal} is trained on the same four environments as RSF, but, lacking an explicit representation of the transition kernel, cannot adapt its successor features to the evaluated dynamics.
Despite training on four environments spanning $p \in \{0.65, 0.95\}$, USFA's error pattern shows a vertical band of low error centered slightly below $p=0.8$ -- close to the average step probability across its training set -- closely resembling a model trained only at a fixed $p=0.8$, as RSF-reward-only is (see below).
This yields a substantially higher error ($\mu=0.131\pm0.069$) than RSF despite identical training data, indicating that decoupling reward and dynamics into robust successor features, rather than the training environments alone, drives RSF's improvement.

\paragraph{RSF-reward-only: Reward-only Ablation}
This ablation is trained on two environments with a fixed step probability $p=0.8$ but differing step rewards of $-0.02$ and $-0.2$, matching the environment coverage of the default USFA setting in \citet{borsa2018universal}.
We observe a vertical band of low error around $p \in [0.78, 0.82]$ across all step rewards; deviating from this interval increases the error substantially ($\mu=0.186\pm0.094$), since the model never observes dynamics variation during training.

\paragraph{RSF-prob-only: Dynamics-only Ablation}
This ablation is trained on two environments with a fixed step reward of $-0.1$ but different step probabilities ($p=0.65$ and $p=0.95$).
We observe low error along the horizontal axis near the training reward of $-0.1$ across all step probabilities; as the step reward deviates from $-0.1$, error increases sharply ($\mu=0.263\pm0.150$, the highest among the four methods), reflecting the model's inability to generalize across reward variation.

\begin{figure}[htbp]
    \centering
	\includegraphics[width=\textwidth]{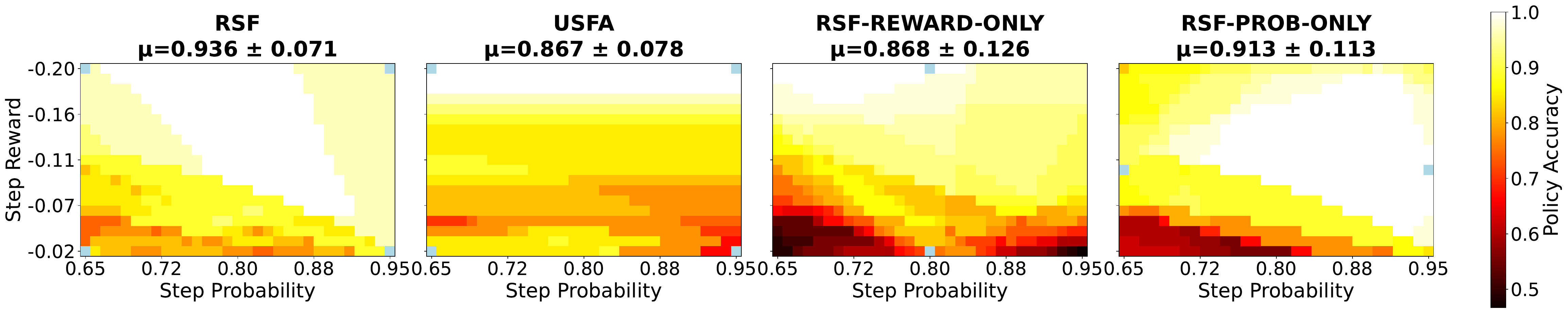}
    \caption{
    	Policy accuracy for the stochastic 2D grid experiments.
    	Brighter regions indicate higher accuracy (better performance).
    	Training environments are marked by the light blue rectangles.
    	The mean accuracies for RSF, USFA, RSF-reward-only, and RSF-prob-only are $93.6\%$, $86.7\%$, $86.8\%$, and $91.3\%$, respectively.}
    \label{fig:aima-experiments-policy-accuracy}
\end{figure}

\paragraph{Policy Accuracy}
Evaluating the resulting policy accuracy (Figure~\ref{fig:aima-experiments-policy-accuracy}) reinforces and nuances these insights. 
Formally, this metric is defined as $\frac{1}{\lvert \mathcal{S} \rvert} \sum_{s \in \mathcal{S}} \mathbb{I}\{\pi^*(s)=\hat{\pi}(s)\}$.
RSF again achieves the highest accuracy ($\mu=93.6\%\pm7.1\%$), followed by RSF-prob-only ($\mu=91.3\%\pm11.3\%$), RSF-reward-only ($\mu=86.8\%\pm12.6\%$), and USFA ($\mu=86.7\%\pm7.8\%$).
Notably, this ranking does not mirror the MAE results: RSF-prob-only has the worst Q-value error of the four methods but the second-best policy accuracy, since inaccurate Q-value magnitudes do not necessarily change the $\arg\max$ action, whereas RSF-reward-only, which never observes dynamics variation during training, yields worse policies despite a lower MAE than RSF-prob-only.
This indicates that explicitly modeling both transition and reward variation, as RSF does, is necessary for consistently accurate zero-shot policies, not merely accurate value estimates.


%% file: 7_discussion.tex
\section{Discussion}

\paragraph{Model-free.}
In this work, we introduced a generalization of the {\em successor features} \citep{barreto2017successor, borsa2018universal} that decouples the quality function not just from the reward but also from the system dynamics. 
We assumed that Markov decision processes (MDPs) can be modeled linearly \citep{yang2019sample}, utilizing potentially distinct feature spaces for rewards and dynamics. 
Since finite MDPs are trivially linearizable with $\phi, \varphi, {\bf w}, {\bf p} \in \mathbb{R}^{|{\cal S}|^2|{\cal A}|}$, a primary question is whether these components can be represented succinctly and learned jointly.
Previous work addressed this question for multi-task learning \citep{caruana1997multitask,baxter2000model} and successor features \citep{barreto2017successor}. 
In our case, the representation learned is for the {\em robust successor features} that predict the outcome for any given task ${\bf w}$ and first-step dynamics ${\bf p}$; our few-shot experiments (Appendix~\ref{sec:apdx-few-shot}) already show that these weights can be estimated from limited interactions, so the open extension is to jointly learn the underlying feature maps $\phi$ and $\varphi$ themselves.

\paragraph{Exploration.}
Although Algorithm~\ref{alg:learn_robust_sf} trains the {\em robust successor features} using a GPI policy, it may still suffer from poor exploration due to its reliance on an $\epsilon$-greedy mechanism. In this regard, entropy-regularized RL \citep{neu2017unified} has been used in the optimization objectives to enhance the exploration and exploit it for zero-shot RL \citep{zisselman2023explore}. One potential model-based regularizer is the {\em Maximum Occupancy Principle} \citep{ramirez2024complex}, where the entropy is defined over both states and actions.

\paragraph{Representation collapse.}
Another consideration is that the loss function is built on a temporal difference (TD) learning method \citep{sutton1988learning,dayan1992convergence}. A limitation of this loss for representation learning is that it may lead to representation collapse, which is something not observed in this work since the representation is given. However, a natural step taken in previous work is to jointly learn the representation, the linear weights and the successor features. Thus, alternative methods (e.g., \citet{chua2024learning}) propose a contrastive loss to mitigate this effect, while improving the efficiency of learning successor features. 

\paragraph{Expressivity.}
Finally, we assume that MDPs have linear reward functions and transition kernels, which may not hold in highly complex environments with continuous state or action spaces, such as robotic control. 
In the context of transfer learning across similar but non-identical dynamical systems, this limitation is often addressed using the Hidden-Parameter MDP framework \citep{doshi2016hidden,killian2017robust}. 
Integrating this framework with robust successor features represents a compelling direction for future research.

%% file: acknowledgements.tex
\section*{Acknowledgements}

This work has been co-funded by MICIU/AEI/UE-PID2023-147145NB-I00, and MCIN/AEI/10.13039/501100011033 under the Maria de Maeztu Units of Excellence Programme (CEX2021-001195-M). Javier Segovia-Aguas is supported by the Ram{\'o}n y Cajal program, RYC2024-050163-I, funded by MICIU/AEI/10.13039/501100011033 and FSE+.

%% file: app_C_generalization.tex
\section{Theoretical Results}
\label{sec:theoretical_results}
We first establish a generalized value-gap bound accounting for differing transition kernels (Prop.~\ref{prop:opt_value_gap}, Lem.~\ref{lem:opt_policy_gap});
we then bound reward differences and value magnitudes via $\phi_{\max}$ (Props.~\ref{prop:rew_up_bound}, \ref{prop:max_reward_diff}, Lem.~\ref{lem:max_q_value});
next, we bound the error of evaluating library policies under the target dynamics with robust SFs (Lem.~\ref{lem:cross_dynamics}, Prop.~\ref{prop:approx_error});
combining these proves Thm.~\ref{thm:gpi_robust_sf}.

\begin{remark}[Notation]
\label{rem:notation}
Throughout, $M_i, M_j \in {\cal M}^{\phi,\varphi}$ denote two arbitrary tasks, with dynamics weights $\mathbf p, \mathbf q$ lying in the probability simplex, and with $Q^i_i \equiv Q^{\pi_i}_i$ the value of $M_i$ under its own optimal policy $\pi_i$, and $Q^j_i \equiv Q^{\pi_j}_i$ the value of policy $\pi_j$ evaluated on $M_i$.
Results up to and including Lem.~\ref{lem:max_q_value} hold for any such pair.
From Lem.~\ref{lem:cross_dynamics} onward, we identify $i$ with the query task $(\mathbf{w,p})$ and $j$ with a reference task $(\mathbf{z,q}) \in {\cal C}$, matching the main text.
\end{remark}

\begin{definition}[Expected immediate features]
\label{def:exp_immediate_features}
\begin{align*}
\phi(s,a) &= \mathbb{E}_{S'\sim {\cal P}(\cdot|s,a)}\left[ \phi(s,a,S') \right] \\
		  &= \sum_{s'} {\cal P}(s'|s,a) \phi(s,a,s').
\end{align*}
When ${\cal P} = P_i$ is the transition kernel of a specific task $M_i$, we write $\phi_i(s,a)$ for the resulting expectation, to make the dependence on the task's dynamics explicit.
\end{definition}

\begin{definition}[Expected immediate reward]
\begin{align*}
r(s,a) &= \mathbb{E}_{S'\sim {\cal P}(\cdot|s,a)}\left[ r(s,a,S') \right] \\
       &= \sum_{s'} {\cal P}(s'|s,a) r(s,a,s') \\
       &= \sum_{s'} {\cal P}(s'|s,a) \phi(s,a,s')^\top {\bf w}\\
       &= \phi(s,a)^\top {\bf w}.
\end{align*}
\end{definition}

\begin{definition}[Maximum reward features]
\label{def:max_reward_features}
Let
$$\phi_{\max} := \max_{s,a,s'} \|\phi(s,a,s')\|.$$
This differs from \citet{barreto2017successor}, who define this quantity as $\max_{s,a}\|\phi(s,a)\|$ for $\phi(s,a)$ the expectation of Def.~\ref{def:exp_immediate_features} under a single, fixed transition kernel.
Since we consider a family of transition kernels, the expectation $\phi_i(s,a)$ now varies by task, so we instead take the maximum over the raw, unweighted features.
Lem.~\ref{lem:phimax_bound} shows that $\phi_{\max}$ still uniformly bounds every task's expected features $\phi_i(s,a)$, so it plays the same role as the original quantity while also being usable directly on raw features where needed (Prop.~\ref{prop:max_reward_diff}).
\end{definition}

\begin{lemma}
\label{lem:phimax_bound}
For any task $M_i \in {\cal M}^{\phi,\varphi}$ with transition kernel $P_i$, and any $s\in{\cal S}, a\in{\cal A}$, $$\|\phi_i(s,a)\| \leq \phi_{\max}.$$
\end{lemma}

{\em Proof.}
\begin{align*}
    \|\phi_i(s,a)\| &= \Big\|\sum_{s'} P_i(s'|s,a)\phi(s,a,s')\Big\| \\
    &\leq \sum_{s'} P_i(s'|s,a)\|\phi(s,a,s')\| &\text{(Triangle ineq.)}\\
    &\leq \phi_{\max}\sum_{s'} P_i(s'|s,a) & \text{(Def.~\ref{def:max_reward_features})}\\
    &= \phi_{\max}.
\end{align*}
$\hfill\square$

\subsection*{Proof of Prop.~\ref{prop:exact_p}}

{\em Proof}.
Transition kernel ${\cal P}$ has dimensionality $|{\cal SAS}|\times 1$ where every entry is mapped to a scalar that denotes the transition probability, and $\varphi$ has $k\times|{\cal SAS}|$ which maps every entry to a column vector of $k$ features.
Entries to both matrices are triplets $(s,a,s')$ where $a$ is the action applied in $s$ that leads to $s'$.
Then, the aim is to learn a vector $\textbf{p}$ of $k$ weights that linearize the transition kernel, so it is straightforward to see that

\begin{align*}
	{\cal P} &= \varphi^\top \textbf{p}\\
	\varphi {\cal P} &= \varphi \varphi^\top \textbf{p} \\
	(\varphi \varphi^\top)^{-1}\varphi {\cal P} &= \textbf{p}
\end{align*} 

assuming that $(\varphi\varphi^\top)$ is invertible and reordering the terms concludes the proof. $\hfill\square$

\begin{proposition}
    \label{prop:p_norm}
    Given {\bf p} $\in\mathbb{R}^k$ such that $\|{\bf p}\|_1 = 1$, then 
    $\|{\bf p}\| \leq 1$.
\end{proposition}

{\em Proof.}
By definition, if $\|{\bf p}\|_1 = 1$, then for every element ${\bf p}_i$ in the vector $ {\bf p}\in \mathbb{R}^k$ we have that $|{\bf p}_i| \leq 1$.
Then, properties like $|{\bf p}_i|^2 \leq |{\bf p}_i|$ for any real number $|{\bf p}_i| \in [0, 1]$, and  $|{\bf p}_i|^2 = {\bf p}^2_i$, are fundamental, resulting in the inequality ${\bf p}_i^2 \leq |{\bf p}_i|$.
Summing both sides of the inequality over all components, we obtain $\sum_i {\bf p}_i^2 \leq \sum_i |{\bf p}_i|$. Note that the right-hand side is the definition of the $l_1$-norm which is $1$ for ${\bf p}$, so $\sum_i {\bf p}_i^2 \leq 1$.
The proof concludes by applying the square root to both sides, that is, $\|{\bf p}\| = \sqrt{\sum_i {\bf p}^2_i} \leq \sqrt{1} = 1$. $\hfill\square$

\begin{definition}[Transition-kernel sensitivity]
\label{def:Lphi}
For the class ${\cal M}^{\phi,\varphi}$, let $$L_\varphi := \max_{s,a} \sum_{s'} \|\varphi(s,a,s')\|,$$ and, for dynamics weights ${\bf p},{\bf q}$, define $$\kappa_{\bf pq} := \min\bigl(2,\; L_\varphi\|{\bf p}-{\bf q}\|\bigr).$$
\end{definition}

\begin{lemma}[Transition-kernel $\ell_1$ bound]
\label{lem:kernel_l1_bound}
For $M_i, M_j \in {\cal M}^{\phi,\varphi}$ with dynamics weights $\mathbf{p},\mathbf{q}$, and any $s \in {\cal S}, a \in {\cal A}$, $$\|P_{\mathbf{p}}(\cdot|s,a) - P_{\mathbf{q}}(\cdot|s,a)\|_1 \leq \kappa_{\mathbf{pq}}.$$
\end{lemma}

{\em Proof.}
On the one hand $\|P_{\mathbf{p}}(\cdot|s,a) - P_{\mathbf{q}}(\cdot|s,a)\|_1 \leq \|P_{\mathbf{p}}(\cdot|s,a)\|_1 + \|P_{\mathbf{q}}(\cdot|s,a)\|_1 = 2$.
On the other,
\begin{align*}
    \|P_{\mathbf{p}}(\cdot|s,a) - P_{\mathbf{q}}(\cdot|s,a)\|_1
    &= \sum_{s'}\left|\varphi(s,a,s')^\top \mathbf{p} - \varphi(s,a,s')^\top \mathbf{q}\right| \\
    &= \sum_{s'} \bigl|\varphi(s,a,s')^\top({\bf p} - {\bf q})\bigr| \\
    &\leq \sum_{s'} \|\varphi(s,a,s')\| \, \|{\bf p} - {\bf q}\| \tag{Cauchy--Schwarz ineq.} \\
    &\leq L_\varphi \|{\bf p} - {\bf q}\|. \tag{Def.~\ref{def:Lphi}}
\end{align*}
Taking the minimum of the two upper bounds concludes this proof.$\hfill\square$

\begin{proposition}
    \label{prop:opt_value_gap}
    Let $p_i$ and $p_j$ be two probability distributions over the state space, and let $V_i$ and $V_j$ be bounded functions over states.
    Then
    \begin{equation}
        |\langle p_i, V_i \rangle - \langle p_j, V_j \rangle|
        \leq \|V_i - V_j\|_\infty + \|p_i - p_j\|_1 \min(\|V_i\|_\infty,\|V_j\|_\infty).
        \label{eq:l1_value_gap}
    \end{equation}
\end{proposition}

{\em Proof.}
\begin{align*}
    \left|\langle p_i, V_i \rangle - \langle p_j, V_j \rangle\right|
    &= \left|\sum_s p_i(s)V_i(s) - p_j(s)V_j(s)\right| \\
    &= \left|\sum_s p_i(s)V_i(s) - p_i(s)V_j(s) + p_i(s)V_j(s) - p_j(s)V_j(s)\right| \\
    &= \Bigl|\sum_s p_i(s)\bigl(V_i(s) - V_j(s)\bigr) + \bigl(p_i(s) - p_j(s)\bigr)V_j(s)\Bigr| \\
    &\leq \left|\sum_s p_i(s)(V_i(s) - V_j(s))\right| + \left|\sum_s (p_i(s)- p_j(s))V_j(s)\right| \\
    &\leq \max_s |V_i(s) - V_j(s)| + \|p_i - p_j\|_1 \|V_j\|_\infty. \tag{H\"{o}lder's ineq.}
\end{align*}
Exchanging the roles of $V_i$ and $V_j$ in the decomposition yields the same bound with $\|V_i\|_\infty$ in place of $\|V_j\|_\infty$; taking the minimum of the two concludes the proof. $\hfill\square$

Note that Prop.~\ref{prop:opt_value_gap} holds uniformly over state-action pairs; e.g., we apply it with $p_i \equiv P_{\bf p}(\cdot|s,a)$ and $V_i(s') \equiv Q^i_i(s',\pi_i(s'))$ (resp., $p_j$ and $V_j$), so that $\|p_i-p_j\|_1 = \|P_{\mathbf p}(\cdot|s,a) - P_{\mathbf q}(\cdot|s,a)\|_1$.
Applying Lem.~\ref{lem:kernel_l1_bound} to bound this quantity by $\kappa_{\bf pq}$ and substituting into Eq.~\ref{eq:l1_value_gap} gives
\begin{equation}
    |\langle p_i, V_i \rangle - \langle p_j, V_j \rangle| \leq \|V_i - V_j\|_\infty + \kappa_{\bf pq} \min(\|V_i\|_\infty,\|V_j\|_\infty).
    \label{eq:kappa_value_gap}
\end{equation}
The coefficient $\kappa_{\bf pq}$ interpolates between the shared-dynamics setting ($\kappa_{\bf pq} \to 0$ as ${\bf q} \to {\bf p}$), where the additional term vanishes, and the model-agnostic worst case ($\kappa_{\bf pq} = 2$, recovered by the trivial bound $\|p_i-p_j\|_1 \le 2$), which matches the bound obtainable without the linear structure of Eq.~\ref{eq:linear_kernel}. 
We use this specialized bound in Lem.~\ref{lem:opt_policy_gap} below.

The following extends Lemma~1 in \citet{barreto2017successor} to tasks with differing dynamics; their result corresponds to the special case $\kappa_{\bf pq} \equiv 0$, obtained under a shared transition kernel.

\begin{lemma}
    \label{lem:opt_policy_gap}
    Given two different MDPs $M_i, M_j \in {\cal M}^{\phi,\varphi}$ with transition dynamics weights $\mathbf{p,q}$, let $\delta_{ij} = \max_{s,a} |r_i(s,a) - r_j(s,a)|$.
    Then,
    \begin{align*}
      &Q^{\pi_i}_i(s,a) - Q^{\pi_j}_i(s,a) \\
      &\quad\leq \frac{1}{(1-\gamma)}\left(2\delta_{ij} + \gamma\kappa_{\bf pq}\min\left(\|Q^{\pi_j}_j\|_\infty,\|Q^{\pi_j}_i\|_\infty\right) + \gamma\kappa_{\bf pq}\min\left(\|Q^{\pi_j}_j\|_\infty,\|Q^{\pi_i}_i\|_\infty\right)\right)  \\
      &\quad\leq \frac{2}{(1-\gamma)}\left(\delta_{ij} + \gamma\kappa_{\bf pq}\|Q^{\pi_j}_j\|_\infty\right).
    \end{align*}
\end{lemma}

{\em Proof.}
Following \citet{barreto2017successor}, the notation is simplified to $Q^j_i(s,a) \equiv Q^{\pi_j}_i(s,a)$. Then,

\begin{align*}
Q^i_i(s,a) - Q^j_i(s,a) &= Q^i_i(s,a) - Q^j_j(s,a) + Q^j_j(s,a)-Q^j_i(s,a) \\
&\leq |Q^i_i(s,a)-Q^j_j(s,a)| + |Q^j_j(s,a)-Q^j_i(s,a)|.
\end{align*}

The proof continues by upper bounding the terms $|Q^i_i(s,a)-Q^j_j(s,a)|$ and $|Q^j_j(s,a)-Q^j_i(s,a)|$, where the underlying MDPs may have different transition kernels and reward functions.
Let $\Lambda_{ij} = \max_{s,a} |Q^i_i(s,a)-Q^j_j(s,a)|$. Then,

\begin{align*}
    &|Q^i_i(s,a) - Q^j_j(s,a)| \\
    &= \left|r_i(s,a) + \gamma \sum_{s'} {\cal P}_i(s'|s,a)Q^i_i(s',\pi_i(s')) 
                                - r_j(s,a) - \gamma \sum_{s'} {\cal P}_j(s'|s,a)Q^j_j(s',\pi_j(s'))\right| \\
    &\leq \left|r_i(s,a) - r_j(s,a)\right| + \left|\gamma \sum_{s'} {\cal P}_i(s'|s,a)Q^i_i(s',\pi_i(s')) - {\cal P}_j(s'|s,a)Q^j_j(s',\pi_j(s'))\right| \\
    &\leq \delta_{ij} + \gamma\left|\sum_{s'} {\cal P}_i(s'|s,a) Q^i_i(s',\pi_i(s')) - {\cal P}_j(s'|s,a) Q^j_j(s',\pi_j(s'))\right|  \\
    &\leq \delta_{ij} + \gamma \left(\Lambda_{ij} + \kappa_{\bf pq}\min(\|Q^i_i\|_\infty,\|Q^j_j\|_\infty)\right). \tag{Eq.~\ref{eq:kappa_value_gap}} \\
\end{align*}

This proves that $\Lambda_{ij} \leq \delta_{ij} + \gamma \left(\Lambda_{ij} + \kappa_{\bf pq}\min(\|Q^i_i\|_\infty,\|Q^j_j\|_\infty)\right)$, hence 
$$
\Lambda_{ij} \leq \frac{1}{(1-\gamma)} \left(\delta_{ij} + \gamma\kappa_{\bf pq}\min(\|Q^i_i\|_\infty, \|Q^j_j\|_\infty)\right).
$$

Now let's follow some similar steps for the second term, $|Q^j_j(s,a)-Q^j_i(s,a)|$.
First, we define $\Lambda'_{ij} = \max_{s,a} |Q^j_j(s,a) - Q^j_i(s,a)|$.
Then,

\begin{align*}
    &|Q^j_j(s,a) - Q^j_i(s,a)| \\
    &= \left| r_j(s,a) + \gamma\sum_{s'}{\cal P}_j(s'|s,a)Q^j_j(s',\pi_j(s')) 
                                      - r_i(s,a) - \gamma\sum_{s'}{\cal P}_i(s'|s,a)Q^j_i(s',\pi_j(s'))\right| \\
    &\leq \left| r_j(s,a) - r_i(s,a)\right| + 
                                    \gamma\left|\sum_{s'} {\cal P}_j(s'|s,a)Q^j_j(s',\pi_j(s')) - 
                                                          {\cal P}_i(s'|s,a)Q^j_i(s',\pi_j(s'))\right|\\
    &\leq \delta_{ij} + \gamma\left(\Lambda'_{ij} + \kappa_{\bf pq}\min(\|Q^j_j\|_\infty,\|Q^j_i\|_\infty)\right). \tag{Eq.~\ref{eq:kappa_value_gap}}
\end{align*}

Similarly, for $\Lambda'_{ij}$, we obtain 
$$
\Lambda'_{ij} \leq \frac{1}{(1-\gamma)} \left(\delta_{ij} + \gamma\kappa_{\bf pq}\min(\|Q^j_j\|_\infty,\|Q^j_i\|_\infty)\right).
$$

Since $Q^i_i(s,a) - Q^j_i(s,a) \leq \Lambda_{ij} + \Lambda'_{ij}$, the proof concludes by merging the known bounds for $\Lambda_{ij}$ and $\Lambda'_{ij}$. $\hfill\square$

\begin{proposition}
    \label{prop:rew_up_bound}
    Let $M_i\in{\cal M}^{\phi,\varphi}$ be an MDP whose reward function is structured as $r_i(s,a) = \phi_i(s,a)^\top {\bf w}_i$ for all states and actions. 
    Then,
    $$\|r_i\|_\infty \leq \phi_{\max}\|{\bf w}_i\|.$$
\end{proposition}

{\em Proof.} 
\begin{align*}
    \|r_i\|_\infty &= \max_{s,a} |r_i(s,a)| \\
    &= \max_{s,a} |\phi_i(s,a)^\top {\bf w}_i| \\
    &\leq \max_{s,a} \|\phi_i(s,a)\|\cdot\|{\bf w}_i\| \tag{Cauchy--Schwarz ineq.}\\
    &\leq \phi_{\max} \|{\bf w}_i\|. \tag{Lem.~\ref{lem:phimax_bound}}
\end{align*}
$\hfill\square$

\begin{proposition}
	\label{prop:max_reward_diff}
	Let $M_i,M_j\in{\cal M}^{\phi,\varphi}$ be two MDPs with reward functions $r_i, r_j$, task weights $\mathbf{w, z}$ and transition kernel weights $\mathbf{p, q}$, respectively.
	Then, the maximum difference in their reward functions in any state-action pair is:
	$$
	\max_{s,a}|r_i(s,a) - r_j(s,a)| \leq \phi_{\max} \|\mathbf{w-z}\| + \kappa_{\mathbf{pq}}\phi_{\max}\min\left(\|\mathbf{w}\|, \|\mathbf{z}\|\right).
	$$
\end{proposition}

{\em Proof}.
\begin{align*}
	&\max_{s,a}|r_i(s,a) - r_j(s,a)| \\
	&=\max_{s,a}\Big|\sum_{s'} P_i(s'|s,a)\phi(s,a,s')^\top\mathbf{w} - \sum_{s'} P_i(s'|s,a)\phi(s,a,s')^\top\mathbf{z} \ +\\ 
	&\qquad\qquad \sum_{s'} P_i(s'|s,a)\phi(s,a,s')^\top\mathbf{z} - \sum_{s'} P_j(s'|s,a)\phi(s,a,s')^\top\mathbf{z}\Big| \\
	&=\max_{s,a}\Big|\sum_{s'} P_i(s'|s,a)\phi(s,a,s')^\top(\mathbf{w - z}) + \\ 
	&\qquad\qquad\sum_{s'} \left(P_i(s'|s,a) - P_j(s'|s,a)\right)\phi(s,a,s')^\top\mathbf{z}\Big| \\
	&\leq \phi_{\max}\|\mathbf{w-z}\| + \max_{s,a}\Big|\sum_{s'} \left(P_i(s'|s,a) - P_j(s'|s,a)\right)\phi(s,a,s')^\top\mathbf{z}\Big| \tag{Prop.~\ref{prop:rew_up_bound}}\\
	&\leq \phi_{\max}\|\mathbf{w-z}\| + \max_{s,a}\sum_{s'} \big|P_i(s'|s,a) - P_j(s'|s,a)\big|\cdot\big|\phi(s,a,s')^\top\mathbf{z}\big| \tag{Triangle ineq.}\\
	&\leq \phi_{\max}\|\mathbf{w-z}\| + \max_{s,a}\Big[\sum_{s'}\big|P_i(s'|s,a) - P_j(s'|s,a)\big| \cdot \max_{s''}\big|\phi(s,a,s'')^\top\mathbf{z}\big|\Big]\\
	&= \phi_{\max}\|\mathbf{w-z}\| + \max_{s,a}\Big[\big\|P_i(\cdot|s,a) - P_j(\cdot|s,a)\big\|_1 \cdot \max_{s''}\big|\phi(s,a,s'')^\top\mathbf{z}\big|\Big]\\
	&\leq \phi_{\max}\|\mathbf{w-z}\| + \kappa_{\mathbf{pq}} \max_{s,a,s''}\big|\phi(s,a,s'')^\top\mathbf{z}\big| \tag{Lem.~\ref{lem:kernel_l1_bound}}\\
	&\leq \phi_{\max}\|\mathbf{w-z}\| + \kappa_{\mathbf{pq}} \phi_{\max} \|\mathbf{z}\| \tag{Cauchy--Schwarz ineq., Def.~\ref{def:max_reward_features}}
\end{align*}
By symmetry, this also holds for expanding with $\sum_{s'} P_j(s'|s,a)\phi(s,a,s')^\top\mathbf{w}$, which gives us the minimum between $\mathbf{w}$ and $\mathbf{z}$ in the last term. $\hfill\square$

\begin{lemma} 
    \label{lem:max_q_value}
    Let $M_i \in \mathcal{M}^{\phi,\varphi}$ with reward weights $\mathbf{w}_i$, and let $\pi$ be any policy.
    Then, $$\|Q^\pi_i\|_\infty \leq \frac{\phi_{\max}\|{\bf w}_i\|}{(1-\gamma)}.$$
    In particular, this holds for the optimal value function $Q^i_i$.
\end{lemma}

{\em Proof.}
\begin{align*}
	\|Q^\pi_i\|_\infty &= \max_{s,a} \Bigl|r_i(s,a) + \gamma \sum_{s'} {\cal P}_i(s'|s,a) \sum_b \pi(b|s')\, Q^\pi_i(s',b)\Bigr| \\
    &\leq \max_{s,a} |r_i(s,a)| + \gamma \max_{s,a} \sum_{s'}{\cal P}_i(s'|s,a) \sum_b \pi(b|s')\, |Q^\pi_i(s',b)| \\
    &\leq \phi_{\max} \|{\bf w}_i\| + \gamma \max_{s,a}  \sum_{s'}{\cal P}_i(s'|s,a) \sum_b \pi(b|s')\, |Q^\pi_i(s',b)|  \tag{Prop.~\ref{prop:rew_up_bound}}\\
    &\leq \phi_{\max} \|{\bf w}_i\| + \gamma \|Q^\pi_i\|_\infty,
\end{align*}
where the last step uses that $\sum_{s'}{\cal P}_i(s'|s,a)\sum_b \pi(b|s')\,|Q^\pi_i(s',b)|$ is a convex combination of entries of $|Q^\pi_i|$, hence at most $\|Q^\pi_i\|_\infty$.
Since $Q^\pi_i$ is bounded, rearranging concludes the proof.$\hfill\square$

\begin{lemma}
\label{lem:cross_dynamics}
Let $M_{\mathbf{w,p}}, M_{\mathbf{w,q}}, M_{\mathbf{z,q}} \in \mathcal{M}^{\phi,\varphi}$ and let $\pi_\mathbf{z,q}$ be the optimal policy for $M_{\mathbf{z,q}}$.
Also let $\tau(s,a,\mathbf{z,q})$ be the solution of Eq.~\ref{eq:tau_explicit} for $\pi_\mathbf{z,q}$ under reference dynamics $\mathbf q$.
Then, $$\max_{s,a}\bigl|Q^{\pi_\mathbf{z,q}}_{\mathbf{w,p}}(s,a) - \mathbf p^\top \tau(s,a,\mathbf{z,q})\,\mathbf w\bigr|
  \le \frac{\gamma\,\kappa_{\mathbf{pq}}\,\phi_{\max}\|\mathbf w\|}{(1-\gamma)^2}.$$
\end{lemma}
Note that optimality of $\pi_\mathbf{z,q}$ is not used, so the bound holds for any policy.

{\em Proof}.
Contracting Eq.~\ref{eq:tau_explicit} with $\mathbf q$ gives the SF Bellman equation (Eq.~\ref{eq:bellman_usf}) under $P_{\mathbf q}$, whose unique solution is $\psi^{\pi_{\mathbf{z,q}}}_\mathbf{q}$, so $\sum_{a'}\pi_{\mathbf{z,q}}(a'|s')\,\mathbf{q}^\top\tau(s',a',\mathbf{z,q})\,\mathbf w = V^{\pi_{\mathbf{z,q}}}_{\mathbf{w,q}}(s')$.
Contracting Eq.~\ref{eq:tau_explicit} with $\mathbf p$ and $\mathbf w$, and using $\varphi(s,a,s')^\top\mathbf p = P_{\mathbf p}(s'|s,a)$, gives
$\mathbf p^\top\tau(s,a,\mathbf{z,q})\,\mathbf w = r_{\mathbf{w,p}}(s,a) + \gamma\sum_{s'}P_{\mathbf p}(s'|s,a)\,V^{\pi_{\mathbf{z,q}}}_{\mathbf{w,q}}(s')$.
Then,
\begin{align*}
	&\max_{s,a}|Q^{\pi_{\mathbf{z,q}}}_{\mathbf{w,p}}(s,a) - \mathbf{p}^\top \tau(s,a,\mathbf{z,q}) \mathbf{w}| \\
	&= \max_{s,a} \left|\left(r_{\mathbf{w,p}}(s,a) + \gamma \sum_{s'} P_{\mathbf{p}}(s'|s,a) V^{\pi_{\mathbf{z,q}}}_{\mathbf{w,p}}(s')\right) - \left(r_{\mathbf{w,p}}(s,a) + \gamma \sum_{s'} P_{\mathbf{p}}(s'|s,a) V^{\pi_{\mathbf{z,q}}}_{\mathbf{w,q}}(s')\right)\right| \\
	&\leq \gamma \; \| Q^{\pi_{\mathbf{z,q}}}_{\mathbf{w,p}} - Q^{\pi_{\mathbf{z,q}}}_{\mathbf{w,q}} \|_\infty \\
	&\leq \frac{\gamma}{1-\gamma}\left(\delta_{\mathbf{pq}} + \gamma\kappa_{\mathbf{pq}} \min(\|Q^{\pi_{\mathbf{z,q}}}_{\mathbf{w,p}}\|_\infty, \|Q^{\pi_{\mathbf{z,q}}}_{\mathbf{w,q}}\|_\infty)\right) \tag{$\Lambda'$ step of Lem.~\ref{lem:opt_policy_gap}} \\
	&\leq \frac{\gamma}{1-\gamma}\left(\kappa_{\mathbf{pq}}\phi_{\max}\|\mathbf{w}\| + \gamma\kappa_{\mathbf{pq}} \min(\|Q^{\pi_{\mathbf{z,q}}}_{\mathbf{w,p}}\|_\infty, \|Q^{\pi_{\mathbf{z,q}}}_{\mathbf{w,q}}\|_\infty)\right) \tag{Prop. \ref{prop:max_reward_diff} for $M_{\mathbf{w,p}}, M_{\mathbf{w,q}}$} \\
	&\leq \frac{\gamma}{1-\gamma}\left(\kappa_{\mathbf{pq}}\phi_{\max}\|\mathbf{w}\| + \frac{\gamma\kappa_{\mathbf{pq}} \phi_{\max} \|\mathbf{w}\|}{1-\gamma} \right) \tag{Lem. \ref{lem:max_q_value}} \\
	&= \frac{\gamma \kappa_{\mathbf{pq}} \phi_{\max} \|\mathbf{w}\|}{(1-\gamma)^2}
\end{align*}
where $\delta_{\mathbf{pq}} := \max_{s,a}|r_{\mathbf{w,p}}(s,a) - r_{\mathbf{w,q}}(s,a)|$.
We can apply the $\Lambda'$ step of Lem.~\ref{lem:opt_policy_gap} in the second inequality since that step does not use optimality of the policy, so it applies to $\pi_{\mathbf{z,q}}$ in $M_{\mathbf{w,p}}$ and $M_{\mathbf{w,q}}$.
$\hfill\square$

\begin{proposition}
\label{prop:approx_error}
Given the task $M_{\mathbf{w,p}} \in \mathcal{M}^{\phi,\varphi}$ and a library of tasks $\mathcal C \subset \mathcal{M}^{\phi,\varphi}$ with approximations
$\tilde Q^{\pi_{\mathbf{z,q}}}_{\mathbf{w,p}}(s,a) = \mathbf p^\top\tilde\tau(s,a,\mathbf{z,q})\,\mathbf w$, let
$\epsilon = \max_{(\mathbf{z,q})\in\mathcal C}\max_{s,a}|Q^{\pi_{\mathbf{z,q}}}_{\mathbf{w,p}}(s,a) - \tilde Q^{\pi_{\mathbf{z,q}}}_{\mathbf{w,p}}(s,a)|$. Then
$$\epsilon \le \frac{\gamma\,\phi_{\max}\|\mathbf w\|}{(1-\gamma)^2}\max_{(\mathbf{z,q})\in\mathcal C}\kappa_{\mathbf{pq}} + \|\mathbf w\|\,\rho_{\mathcal C},$$
where $\rho_{\mathcal C} = \max_{(\mathbf{z,q})\in\mathcal C}\max_{s,a}\|\tau(s,a,\mathbf{z,q}) - \tilde\tau(s,a,\mathbf{z,q})\|_2$,
$\|\cdot\|_2$ is the spectral norm, and $\tau$ is as in Lem.~\ref{lem:cross_dynamics}.
\end{proposition}

{\em Proof}.

\begin{align*}
    \epsilon &= \max_{\mathbf{(z,q)} \in \mathcal{C}}\max_{s,a} |Q^{\pi_{{\bf z,q}}}_{\bf w,p}(s,a) - \tilde{Q}^{\pi_{{\bf z,q}}}_{{\bf w,p}}(s,a)| \\
    &\leq \max_{\mathbf{(z,q)} \in \mathcal{C}}\max_{s,a} |Q^{\pi_{{\bf z,q}}}_{\bf w,p}(s,a) - \mathbf{p}^\top \tau(s,a,\mathbf{z,q}) \mathbf{w}| + \max_{\mathbf{(z,q)} \in \mathcal{C}}\max_{s,a} |\mathbf{p}^\top \tau(s,a,\mathbf{z,q}) \mathbf{w} - \mathbf{p}^\top \tilde\tau(s,a,\mathbf{z,q}) \mathbf{w}| \\
    &=: \epsilon_1 + \epsilon_2
\end{align*}
Applying Lem. \ref{lem:cross_dynamics} to each $\mathbf{(z,q)} \in \mathcal{C}$ and maximizing over $\mathcal{C}$ bounds $\epsilon_1$.
$\epsilon_2$ is bounded below.

\begin{align*}
	\epsilon_2 &= \max_{\mathbf{(z,q)} \in \mathcal{C}} \max_{s,a} |{\bf p}^\top \tau(s,a,{\bf z},{\bf q}){\bf w} - {\bf p}^\top\tilde{\tau}(s,a,{\bf z},{\bf q}){\bf w}| \\
    &= \max_{\mathbf{(z,q)} \in \mathcal{C}} \max_{s,a} |{\bf p}^\top (\tau(s,a,{\bf z},{\bf q}) - \tilde{\tau}(s,a,{\bf z},{\bf q})){\bf w}| \\
    &\leq \max_{\mathbf{(z,q)} \in \mathcal{C}} \max_{s,a} \|{\bf p}\|\cdot \|\tau(s,a,{\bf z},{\bf q}) - \tilde{\tau}(s,a,{\bf z},{\bf q})\|_2 \cdot \|{\bf w}\| \tag{Cauchy--Schwarz ineq.}\\
    &= \|{\bf p}\|\cdot \|{\bf w}\| \cdot \max_{\mathbf{(z,q)} \in \mathcal{C}} \max_{s,a} \|\tau(s,a,{\bf z},{\bf q}) - \tilde{\tau}(s,a,{\bf z},{\bf q})\|_2  \\
    &\leq \|{\bf w}\| \max_{\mathbf{(z,q)} \in \mathcal{C}} \max_{s,a} \|\tau(s,a,{\bf z},{\bf q}) - \tilde{\tau}(s,a,{\bf z},{\bf q})\|_2  \tag{Prop.~\ref{prop:p_norm}}\\
    &= \|{\bf w}\| \, \rho_{\mathcal{C}}
\end{align*}

Note that we can apply Prop.~\ref{prop:p_norm} since $\mathbf{p}$ lies in the probability simplex (Rem.~\ref{rem:notation}).
Combining both bounds concludes the proof.$\hfill\square$

\subsection*{Proof of Theorem~\ref{thm:gpi_robust_sf}}

{\em Proof.}
Write $i \equiv ({\bf w,p})$ for the new task and index library tasks by $j \equiv ({\bf z,q}) \in {\cal C}$, with $Q^j_j \equiv Q^{\pi_j}_j$ the optimal value function of $M_j$, and let $\epsilon$ be as in Prop.~\ref{prop:approx_error}. 
Then,
\begin{align*}
    &Q^*_i(s,a) - Q^\pi_i(s,a) \leq Q^*_i(s,a) - \max_{(\mathbf{z,q})\in{\cal C}} Q^j_i(s,a) + \frac{2}{1-\gamma}\epsilon \tag{Thm.~1, \citet{barreto2017successor}}\\
    &= \min_{(\mathbf{z,q}) \in \mathcal{C}} \left(Q^*_i(s,a) - Q^j_i(s,a)\right) + \frac{2}{1-\gamma}\epsilon\\
    &\leq \frac{2}{1-\gamma} \left(\min_{(\mathbf{z,q}) \in \mathcal{C}} \left(\delta_{ij} + \gamma\kappa_{\bf pq}\|Q^j_j\|_\infty \right) + \epsilon\right)  \tag{Lem.~\ref{lem:opt_policy_gap}} \\
    &\leq \frac{2}{1-\gamma}\left(\min_{(\mathbf{z,q}) \in \mathcal{C}} \left(\delta_{ij} + \frac{\gamma}{1-\gamma} \kappa_{\bf pq}\phi_{\max} \|{\bf z}\| \right) + \epsilon\right) \tag{Lem.~\ref{lem:max_q_value}} \\
    &\leq \frac{2}{1-\gamma} \left(\min_{(\mathbf{z,q}) \in \mathcal{C}} \left(\phi_{\max} \|{\bf w} - {\bf z}\| + \kappa_{\mathbf{pq}} \phi_{\max} \|\mathbf{z}\| + \frac{\gamma}{1-\gamma} \kappa_{\bf pq}\phi_{\max}  \|{\bf z}\|\right) + \epsilon\right) \tag{Prop.~\ref{prop:max_reward_diff}}\\
    &= \frac{2}{1-\gamma}\left(\phi_{\max} \min_{(\mathbf{z,q}) \in \mathcal{C}} \left(\|{\bf w} - {\bf z}\| + \frac{\kappa_{\bf pq}}{1-\gamma} \|{\bf z}\|\right) + \epsilon \right)\\
    &\leq \frac{2}{1-\gamma}\Biggl(\phi_{\max}\min_{(\mathbf{z,q}) \in \mathcal{C}} \Bigl(\|{\bf w} - {\bf z}\| + \frac{\kappa_{\bf pq}}{1-\gamma}\|{\bf z}\|\Bigr) \\
    &\qquad+ \frac{\gamma\,\phi_{\max}\|\mathbf w\|}{(1-\gamma)^2}\max_{(\mathbf{z,q})\in\mathcal C}\kappa_{\mathbf{pq}} + \|\mathbf w\|\,\rho_{\mathcal C} \Biggr) \tag{Prop.~\ref{prop:approx_error}}
\end{align*}
$\hfill\square$

In conclusion, the gap between the GPI policy and optimal policy for a new task is upper-bounded by four terms.
The first measures the distance between the reward weights of the reference task and the new task, and the second the distance of the transition kernels, scaled by the reference task ``size''.
The third term is the cross-dynamics error of robust SFs (Lem.~\ref{lem:cross_dynamics}): even with exact $\tau$, $\mathbf{p}^\top \tau \, \mathbf{w}$ follows the target dynamics only for the first transition, and this error is governed by the library task whose dynamics are furthest from $\mathbf{p}$.
The last term is the approximation error of $\tilde\tau$. 
As justified in \citet{barreto2017successor}, proving this upper bound is a stronger guarantee, quantifying the loss w.r.t. the optimal policy in the new task instead of comparing the performance of $\pi$ with the previous computed policies.

%% file: 6_related_work.tex
\section{Related Work}
\label{sec:related_work}

Previous approaches to transfer in RL have heavily relied on state abstractions or reward-centric generalizations.
For instance, \citet{lehnert2020transfer} propose computing a set of reward-predictive representations for model reduction that induce a state partition based on bisimulation. While this yields model features that can be transferred across varying reward and transition functions, these abstract spaces are applied solely to states rather than transitions, utilizing the same abstraction for both functions. More closely related to our approach are methods leveraging Successor Features (SFs) and Generalized Policy Improvement (GPI). 
\citet{barreto2018transfer} show how the SF\&GPI framework \citep{barreto2017successor} can generalize to arbitrary reward functions beyond simple linear combinations, formalizing the feature learning as a supervised problem. 
Similarly, \citet{ma2020universal} explored a combination of goal-conditioned RL with SFs to transfer approximations across different goals. 
However, a critical limitation of both approaches is their reliance on shared underlying dynamics; generalizing across goals or arbitrary rewards is fundamentally a variation of reward transfer, rendering them inapplicable to environments with shifting transition kernels.

To bypass some limitations of standard SFs, alternative methods like \citet{ingebrand2024zero} have demonstrated zero-shot capabilities using combinations of non-linear basis functions, termed function encoders. These are represented with a vector of learned coefficients that, in addition to states and actions, serve as inputs to RL algorithms. While this framework assumes the availability of appropriate data to learn the encoding in a supervised manner and has been evaluated on tasks such as predicting transition functions and multi-task RL transfer, it does not explicitly study generalization under combined perturbations in both the transition and reward functions.

Finally, handling variations in both rewards and dynamics has been explored extensively within Multitask Offline RL. \citet{ishfaq2024offline} analyzed this context by gathering input data from various behavior policies across a set of tasks to learn a shared representation, assuming the tasks are low-rank MDPs. Other notable representation learning methods in this space include TD-JEPA \citep{bagatella2026td}, which leverages latent-predictive representations for zero-shot reinforcement learning, and HILP \citep{park2024foundation}, which learns foundation policies using Hilbert representations from offline datasets. While these works offer transfer guarantees for different rewards and transition kernels, they tend to be more computationally expensive than SFs due to mechanisms like low-rank factorization with pessimistic regularization or complex latent predictive models. In contrast, SFs are efficiently trained using standard TD-learning and offer superior zero-shot transfer capabilities for novel tasks under the stronger assumption of linear reward functions.

%% file: app_D_few_shot_experments.tex
\section{Few-Shot Learning}
\label{sec:apdx-few-shot}

This section investigates the few-shot adaptation capabilities of our approach, assessing how quickly the models can recover an optimal policy in novel, unseen environments where the true reward weights $\textbf{w}$ and transition dynamics parameters $\textbf{p}$ are initially unknown to the agent. 
Unlike the zero-shot setting, the agent must estimate these parameters dynamically from limited interactions with the environment. 

We compare the USFA, Robust RL, and Robust SF settings against a standard Double DQN (DDQN) baseline trained entirely from scratch. 
The evaluation metric is the cumulative average return over the first $50$ steps (Figure~\ref{fig:aima-experiments-returns-cum}). 
This isolates the sample efficiency of leveraging a generalized model versus learning a distinct model for every new task.

\begin{figure}[htbp]
    \centering
    \includegraphics[width=0.9\textwidth]{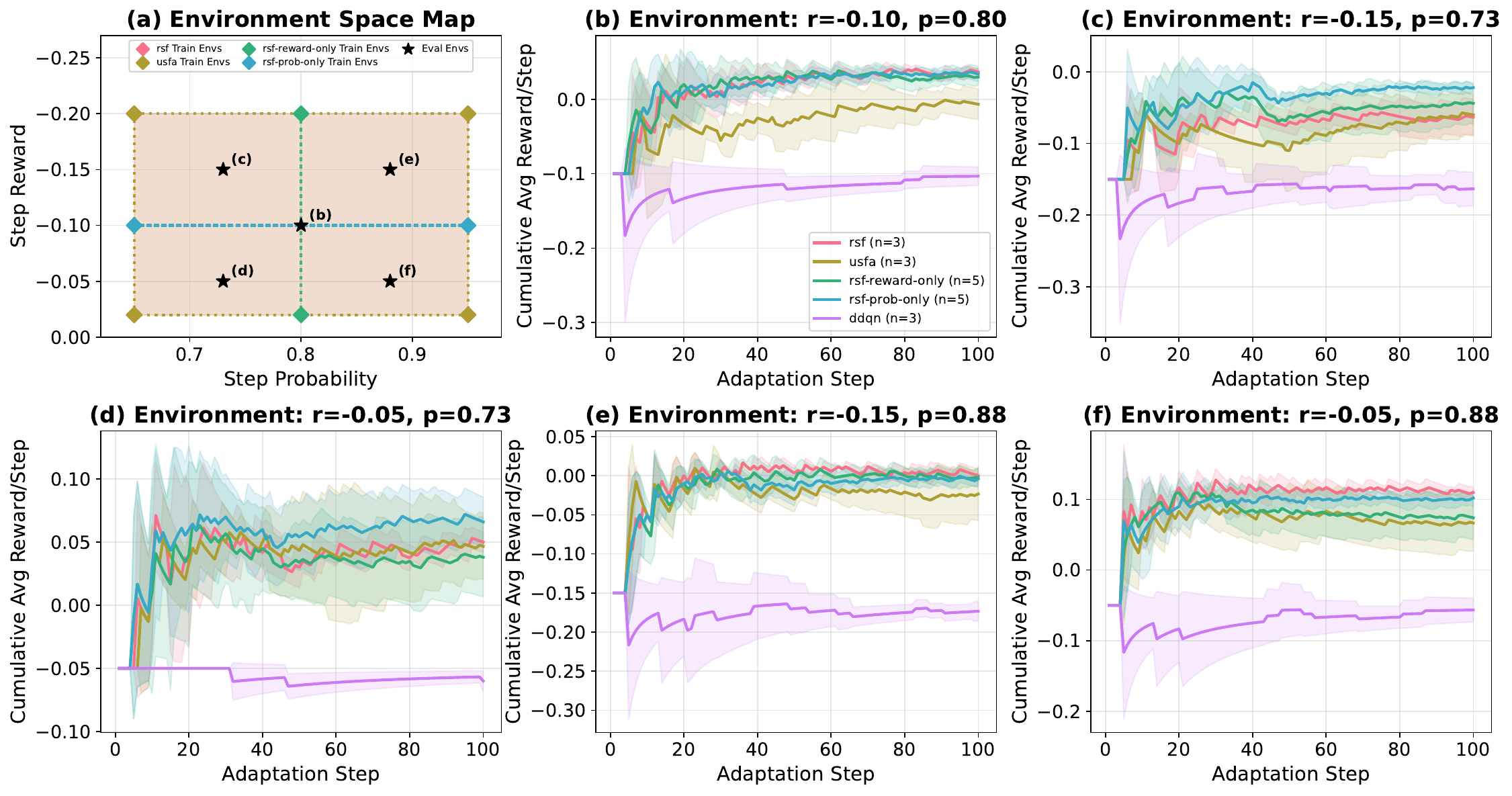}

    \caption{Cumulative average returns during few-shot adaptation in unseen environments. Robust SFs rapidly adapt by estimating $\textbf{w}$ and $\textbf{p}$ from limited interactions, substantially outperforming the DDQN baseline trained from scratch.}
    \label{fig:aima-experiments-returns-cum}
\end{figure}

As shown, robust successor features exhibit rapid adaptation across all unseen evaluation environments, while the USFA and Robust RL settings only adapt to some of them.
The baseline DDQN expectedly shows minimal improvement in this restricted step range.
By leveraging the generalized model $\tilde{\tau}$ captured during the initial training phase, the agent can rapidly estimate the new dynamics and reward structures, yielding near-optimal returns in a fraction of the time required by the DDQN baseline.

Ultimately, this fast few-shot adaptation opens the door to addressing more complex, high-dimensional tasks, such as those found in robotics, where the underlying representations $\phi$ and $\varphi$ must also be learned simultaneously.

All the results reported in this work, including the zero-shot experiments, use the scalar TD-error (Eq.~\ref{eq:td_error}) to learn $\tau$/$\psi$.
Extending this to the vector (or matrix) TD-error is ongoing work.